\documentclass[10pt,twocolumn,letterpaper]{article}

\usepackage{iccv}              
\usepackage{colortbl}
\usepackage{xcolor}
\usepackage{multirow}

\definecolor{iccvblue}{rgb}{0.21,0.49,0.74}
\definecolor{lightblue}{RGB}{220, 230, 255}
\usepackage[pagebackref,breaklinks,colorlinks,allcolors=iccvblue]{hyperref}

\def\paperID{3618} 
\def\confName{ICCV}
\def\confYear{2025}

\title{FW-Merging: Scaling Model Merging with Frank-Wolfe Optimization\thanks{
Accepted at ICCV 2025
}}

\author{
Hao (Mark) Chen$^{1}$ \quad
Shell Xu Hu$^{2}$ \quad
Wayne Luk$^{1}$ \quad
Timothy Hospedales$^{2}$ \quad
Hongxiang Fan$^{1}$ 
\\
$^1$Imperial College London, UK
$^2$Samsung AI Center, Cambridge, UK
\\
{\small 
\texttt{hc1620@ic.ac.uk} \quad 
\texttt{shell.hu@samsung.com} \quad 
\texttt{w.luk@imperial.ac.uk} \quad \texttt{t.hospedales@samsung.com}
}
\\
{\small
\texttt{hongxiang.fan@imperial.ac.uk} 
\thanks{
Corresponding Authors: Shell Xu Hu and Hongxiang Fan}
}
\vspace{-6mm}
}

\usepackage[ruled, lined, linesnumbered, commentsnumbered, longend]{algorithm2e}
\SetKwInOut{KwIn}{Input}
\SetKwInOut{KwOut}{Output}

\usepackage{xcolor}

\newcommand{\fwta}{\texttt{$\text{FW}_{\text{hard}}$}\xspace }
\newcommand{\fwam}{\texttt{$\text{FW}_{\text{soft}}$}\xspace }
\newcommand{\fwm}{\texttt{FW-Merging}\xspace }

\usepackage{amsmath}    
\usepackage{amssymb}    
\usepackage{mathtools}  
\usepackage{amsthm}

\DeclareMathOperator*{\argmin}{arg\,min}

\usepackage{algorithmic}

\newtheorem{proposition}{Proposition}
\theoremstyle{definition}
\newtheorem{theorem}{Theorem}
\begin{document}

\newif\ifhongxiang
\hongxiangtrue

\newif\ifmark
\markfalse

\maketitle
\begin{abstract}
{
Model merging has emerged as a promising approach for multi-task learning (MTL),
offering a data-efficient alternative to 
conventional fine-tuning. However, with the rapid development of the open-source AI ecosystem and the increasing availability of fine-tuned foundation models, existing model merging methods face two key limitations: (i) They are primarily designed for in-house fine-tuned models, making them less adaptable to diverse model sources with partially \textbf{unknown} model and task information, (ii) They struggle to scale effectively when merging \textbf{numerous} model checkpoints.
To address these challenges, we formulate model merging as a constrained optimization problem and introduce a novel approach: \textbf{Frank-Wolfe Merging} (\fwm). 
Inspired by Frank-Wolfe optimization, our approach iteratively selects the most relevant model in the pool to minimize a linear approximation of the objective function and then executes a local merging similar to the Frank-Wolfe update. 
More importantly, \fwm serves as an orthogonal technique for existing merging methods, seamlessly integrating with them to further enhance accuracy performance.
Our experiments show that \fwm scales across diverse model sources, remaining stable with 16 irrelevant models and improving by \textbf{15.3\%} with 16 relevant models on 20 CV tasks, while maintaining constant memory overhead—unlike the linear overhead of data-informed merging methods.
Compared with the state-of-the-art approaches, \fwm surpasses the data-free merging method by \textbf{32.8\%} and outperforms the data-informed Adamerging by \textbf{8.39\%} when merging 20 ViT models. 
Our code is open-sourced at \href{https://github.com/hmarkc/FW-Merging}{here}.
}
\end{abstract}    
\section{Introduction}
\label{sec:intro}

\begin{figure*}[h]
    \centering
    \begin{subfigure}{0.33\textwidth}
        \centering
        \includegraphics[width=\linewidth]{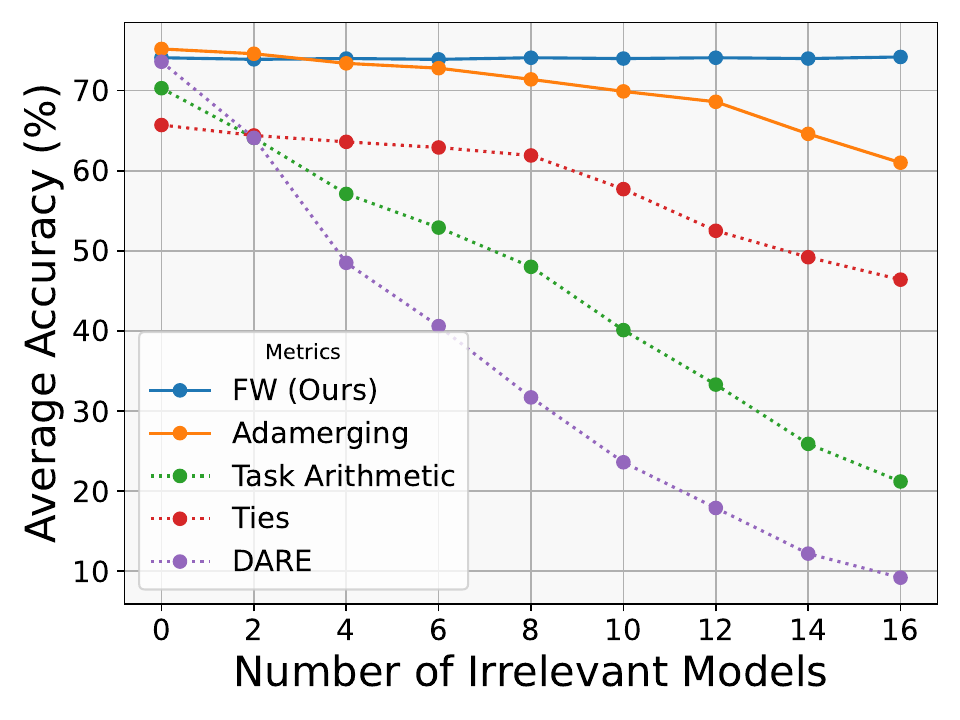}
        \caption{4 CV tasks with irrelevant model addition.}
        \label{fig:scaling1}
    \end{subfigure}
    \hfill
    \begin{subfigure}{0.33\textwidth}
        \centering
        \includegraphics[width=\linewidth]{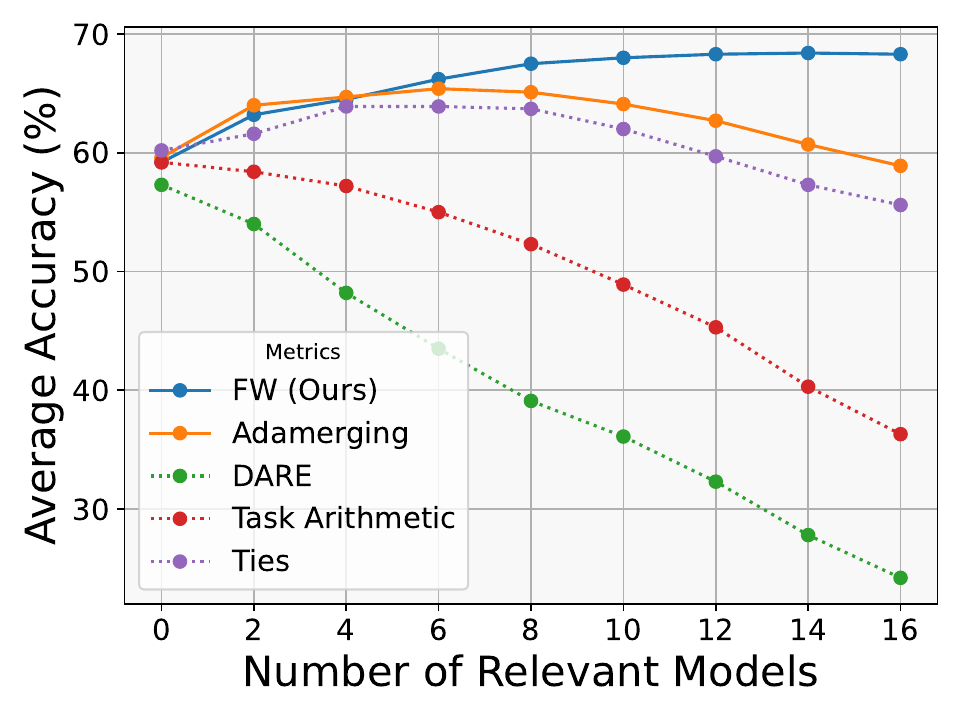}
        \caption{20 CV tasks with relevant model addition.}
        \label{fig:scaling2}
    \end{subfigure}
    \hfill
    \begin{subfigure}{0.33\textwidth}
        \centering
        \includegraphics[width=\linewidth]{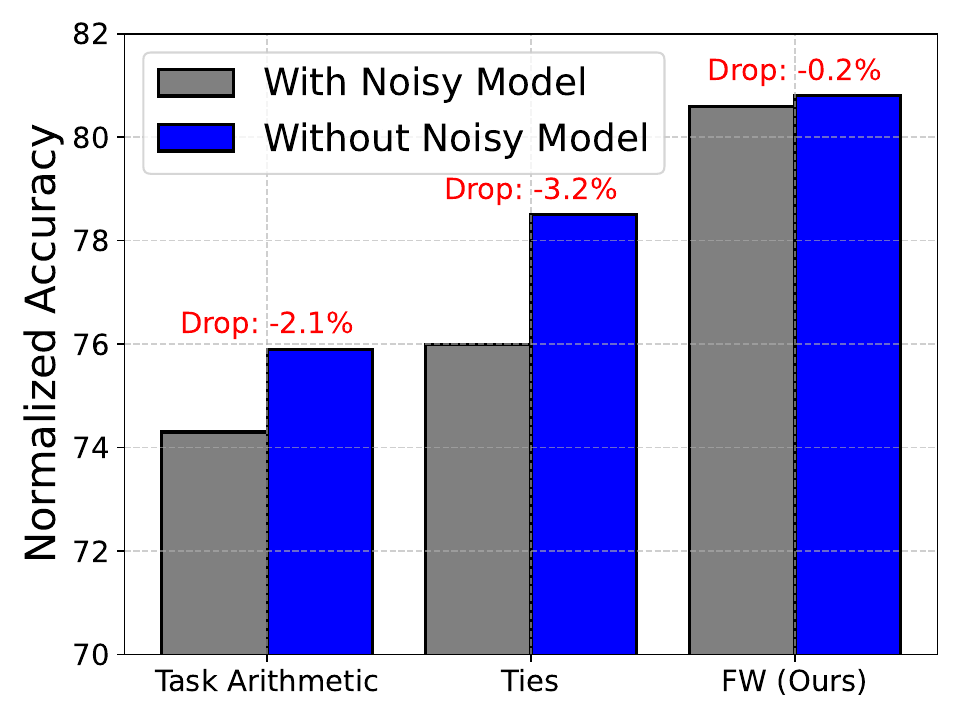}
        \caption{NLP tasks with noisy model addition.}
        \label{fig:scaling3}
    \end{subfigure}
    \caption{Performance scaling of \fwm across CV tasks. (a) demonstrate robustness to irrelevant models, while (b) show improved performance with relevant models. (c) analyzes performance degradation when incorporating a noisy model initialized from a different pre-trained checkpoint. Detailed results and experimental setup are discussed in Section~\ref{exp:scaling}.
    \vspace{-5mm}}
    \label{fig:scaling}
\end{figure*}

\ifmark
Large Language Models (LLMs) achieve excellent performance on downstream tasks after fine-tuning~\cite{li2021prefix, hu2021lora, lester2021power}. However, for multiple downstream tasks, individual fine-tuning leads to significant deployment cost. While an alternative is to fine-tune a single model with multi-task learning (MTL), it demands a large amount of high-quality data, which might only exist in the private domain~\cite{mammen2021federated}, and significant compute resources~\cite{rajbhandari2020zero}. To mitigate this issue, model merging has emerged as a promising technique for fusing fine-tuned models within the parameter space~\cite{ilharco2022editing, yadav2023ties, yang2023adamerging}. Model merging methods can be broadly categorized into two types: data-free methods~\cite{ilharco2022editing, yadav2023ties, izmailov2018averaging}, which do not require additional training data, and test-time adaptation methods~\cite{yang2023adamerging, yang2024representation}, which adjust merge coefficients dynamically during inference. When applied to high-quality fine-tuned checkpoints with a shared initialization, both have proven effective in preserving multi-task capabilities.

Thanks to the thriving open-source ecosystem, including platforms like Hugging Face, and the availability of powerful pre-trained models such as LLaMA~\cite{touvron2023llama} and DeepSeek~\cite{liu2024deepseek}, recent years have seen an exponential increase in the release of fine-tuned models~\cite{kukreja2024literature}. This rapid growth has created a large, diverse pool of models, providing a rich resource for advancing the field of model merging. By leveraging this growing collection of fine-tuned models, model merging has the potential to scale beyond small-scale experiments and transform into a systematic approach for building more robust, versatile LLMs. Instead of relying solely on expensive fine-tuning, large-scale model merging can enable efficient multi-task learning (MTL) with minimal additional cost, while utilizing existing models. Conceptually, the optimization of large-scale model merging can also be viewed as a novel fine-tuning paradigm where the inputs are model checkpoints rather than raw data. Since model checkpoints contain significantly more dense information, and the optimization space of merging coefficients is inherently smaller than that of fine-tuning a full model, large-scale model merging has the potential to offer a more efficient and targeted alternative to traditional MTL.

However, current model merging techniques are primarily designed for small-scale applications, focusing on in-house fine-tuned checkpoints where complete information is available. These methods adjust merging coefficients based on the known capabilities of the models to optimize performance. When scaling up to a large number of unknown models, these methods struggle and can result in significant performance degradation, ranging from 19.3\% to 64.4\% as shown in Figure~\ref{fig:scaling1}. This is mainly due to two factors: 1) the interference of parameters from a large number of models, and 2) the uncertainty of each model's capabilities, which results in both high-quality and poor models being treated equally during the merging process. For example, methods such as DOGE~\cite{wei2025modeling} and Concrete Merging~\cite{tang2023concrete} attempt to find a common subspace for all models. However, these methods fail when bad fine-tuned models are included, as they are not equipped to handle the noise introduced by less relevant models.

To effectively scale up model merging and leverage the vast collection of open-sourced models with unknown capabilities, a new model merging method is required. This new method must exhibit two fundamental scaling properties: \textit{1) as more irrelevant models are added to the merging pool, the performance should remain unaffected}, and \textit{2) as more relevant models are added to the merging pool, the performance should steadily increase, converging towards the optimal performance}. 
To this end, we revisit model merging and formulate it as a constrained optimization problem, where the objective function dictates the desirable behavior of the final merged model, and fine-tuned checkpoints form the constraint set. To solve this problem, we draw inspiration from the famous Frank-Wolfe optimization and adapt it to model merging, introducing \textbf{Frank-Wolfe Merging} (\fwm). 
\fwm is an iterative algorithm that comprises three principal stages in each iteration: (1) \textbf{Relevance Evaluation}: Instead of merging models arbitrarily, we obtain the linear approximation of the objective function using gradients of the current model, revealing the most beneficial direction for improvement. (2) \textbf{Model Selection}: The most relevant checkpoints are selected from the constraint set by minimizing the linear approximation, ensuring that each step incorporates task-specific knowledge with minimal interference. (3) \textbf{Knowledge Integration}: The selected checkpoint is integrated using an orthogonal merging method, striking a balance between adaptation and stability in the merged model.

We demonstrate the effectiveness of \fwm with a diverse pool of fine-tuned checkpoints across various language and vision tasks, compared to both data-free and data-driven model merging methods as well as traditional MTL. As shown in Figure~\ref{fig:scaling}, \fwm satisfy our two foundamental scaling properties: performance does not drop when 16 irrelevant models are added (compared to a 49\% drop in task-arithmetic) and improves by 8.7\% when 16 relevant models are included. \fwm outperforms state-of-the-art data-free merging method by 32.8\% and the data-driven method Adamerging by 8.39\% when merging 20 ViT models. On the language benchmarks, \fwm achieves 6.3\% improvement over the best model merging method across discriminative and generative tasks, while even surpassing the performance of traditional MTL using only 3.4\% of its data. Our results show position \fwm as an effective solution to scale model merging to the next level. 

Our contributions can be summarized as follows:

\begin{itemize}
    \item We frame model merging as a constrained optimization problem with an objective function that explicitly captures desirable behavior.
    \item We introduce Frank-Wolfe Merging, a novel iterative method that optimizes the merged model towards its optimal point, enabling better scalability with large sets of black-box open-source checkpoints.
    \item We evaluate our method on extensive benchmarks, demonstrating its effectiveness and scalability.
\end{itemize}

\else 

Multi-task learning (MTL)-based fine-tuning adapts a single pre-trained Large Language Model (LLM) for multiple downstream applications, reducing the deployment overhead of separately fine-tuning multiple models~\cite{yu2024unleashing}. However, it still demands a large amount of high-quality data, which might only exist in the private domain~\cite{mammen2021federated}, and significant compute resources~\cite{rajbhandari2020zero}. To mitigate these issues, model merging has emerged as a promising technique for fusing fine-tuned models within the parameter space~\cite{ilharco2022editing, yadav2023ties, yang2023adamerging}. 
Existing model merging methods can be broadly classified into two categories: data-free methods~\cite{ilharco2022editing, yadav2023ties, izmailov2018averaging}, and data-informed methods~\cite{yang2023adamerging, yang2024representation}, which optimize merge coefficients based on additional data.

While these approaches have proven effective, several key limitations hinder their scalability and broader adoption.
First, these methods adjust merging coefficients based on the known capabilities of the models on specific tasks to optimize performance, which is less robust when dealing with diverse model sources with unknown information\footnote{{This paper refers unknown information to: \textit{1)} when open-source models are partially assessed on limited benchmarks, leaving their performance on other tasks unknown and costly to evaluate, and \textit{2)} when models contain misleading information, polluting the merging process.}}. This is primarily caused by the inability to distinguish high-quality models from poorly fine-tuned ones in an unknown model setting.
Second, when scaling up these approaches to a large number of unknown models, these methods struggle and can result in significant performance degradation.
As shown in Figure~\ref{fig:scaling1}, our profiling study demonstrates a performance reduction ranging from 18.9\% to 64.4\%.
These limitations are \textbf{further amplified} by the fast-growing open-source AI ecosystem, where platforms such as Hugging Face have driven a surge in the release of powerful LLMs with many lacking complete information.
{Given that merging open-source models has repeatedly shown the potential to produce top-ranking LLMs on major model leaderboards~\cite{huggingface_openllm_2024}, developing scalable and robust merging techniques is essential to harness the growing number of open-source models, further enhancing performance and widening the adoption of model merging.}


 To effectively scale model merging and leverage the vast collection of open-sourced models with unknown capabilities, the new model merging method must exhibit two fundamental scaling properties: \textit{1) as more irrelevant models are added to the merging pool, the performance should remain unaffected}, and \textit{2) as more relevant models are added to the merging pool, the performance should steadily increase, converging towards the optimal performance}. 
To this end, we revisit model merging and formulate it as a constrained optimization problem, where the objective function dictates the desirable behavior of the final merged model, and fine-tuned checkpoints form the constraint set. 
Inspired by Frank-Wolfe optimization, we introduce \textbf{Frank-Wolfe Merging} (\fwm), an iterative algorithm designed to enhance merging efficiency while maintaining robustness at scale.
\fwm comprises three principal stages in each iteration: (1) \textbf{Relevance Evaluation}: Instead of merging models arbitrarily, we obtain the linear approximation of the objective function using gradients of the current model, revealing the most beneficial direction for improvement. (2) \textbf{Model Selection}: The most relevant checkpoints are selected from the constraint set by minimizing the linear approximation, ensuring that each step incorporates task-specific knowledge with minimal interference. (3) \textbf{Knowledge Integration}: The selected checkpoint is integrated using an orthogonal merging method, striking a balance between adaptation and stability in the merged model.

We demonstrate the effectiveness of \fwm with a diverse pool of fine-tuned checkpoints across various language and vision tasks, compared to both data-free and data-informed model merging methods as well as traditional MTL-based fine-tuning. As shown in Figure~\ref{fig:scaling}, \fwm satisfy our two fundamental scaling properties: accuracy performance does not drop when 16 irrelevant models are added (compared to a 49\% drop in task-arithmetic) and steadily improves by 15.3\% when 16 relevant models are included. 
{Additionally, \fwm requires only constant memory overhead, as it selects and merges a fixed number of models at a time. In contrast, methods that optimize merging coefficients~\cite{yang2023adamerging} or resolve parameter interference~\cite{yadav2023ties} must store all models in memory, leading to linear overhead.}
Moreover, \fwm exhibits greater robustness to noisy models lacking critical information, such as their initialization point. As shown in Figure~\ref{fig:scaling3}, \fwm experiences minimal performance degradation when a misinitialized model is introduced, whereas Ties suffers a performance drop of up to 3.2\%.
\fwm outperforms state-of-the-art data-free merging method by \textbf{32.8\%} and the data-informed method Adamerging by \textbf{8.39\%} when merging 20 ViT models. On the language benchmarks, \fwm achieves \textbf{6.3\%} improvement over the best model merging method across discriminative and generative tasks, while even surpassing the performance of traditional MTL using only \textbf{3.4\%} of its required data. Our results position \fwm as an effective solution to scale model merging to the next level. 

{
Our contributions can be summarized as follows:
\begin{itemize}
    \item Identify scalability and robustness issues in existing model merging techniques through experiments, highlighting the urgent need for large-scale model merging.  
    \item Formulate model merging as a constrained optimization problem with an objective function that explicitly captures the desired behavior of the final merged model.  
    \item Introduce Frank-Wolfe Merging, a novel iterative method that autonomously guides the merged model toward an optimized direction, even with large sets of black-box open-source checkpoints.  
    \item Evaluate our proposed approach on extensive benchmarks, demonstrating its effectiveness and scalability.  
\end{itemize}
}

\fi
\section{Related Work} \label{sec:formatting}

\paragraph{Efficient Multi-Task Learning.}
In traditional Multi-Task Learning (MTL), a single model is trained on a dataset containing multiple tasks to enable the model to acquire diverse capabilities~\cite{caruana1997multitask}. However, a significant challenge in traditional MTL is the issue of negative transfer~\cite{jiang2023forkmerge}. To mitigate this, architecture-based approaches have been developed, such as parameter sparsification~\cite{liu2019end, sun2020adashare} and shared structure modularization~\cite{ma2018modeling, ma2019snr}. On the optimization side, methods to resolve gradient conflicts~\cite{yu2020gradient, chen2020just} and domination of gradient or learning rate~\cite{chen2018gradnorm, liu2021towards} have been proposed. With the rise of Large Language Models (LLMs), MTL faces additional challenges, particularly the high computational costs. To address these challenges, strategies like parameter-efficient fine-tuning~\cite{li2021prefix, hu2021lora, lester2021power} and memory-efficient fine-tuning~\cite{dettmers2024qlora, li2023loftq, malladi2023fine} have been introduced to minimize both memory and computational resource usage. More recently, model merging has emerged as a promising approach to make MTL more compute- and data-efficient.

\paragraph{Model Merging.}
While pre-merging methods prepare favorable conditions for merging, during-merging techniques combine multiple neural networks into a single model while retaining or enhancing their capabilities~\cite{yang2408model}.
In this work, we focus on during-merging methods.
Early insights into neural network landscapes~\cite{goodfellow2014qualitatively} revealed that linear interpolation between models exposes useful loss surface properties, laying the foundation for weight averaging—a core merging technique. Simple averaging widens optima and improves generalization~\cite{izmailov2018averaging}, evolving into advanced methods like model soups~\cite{wortsman2022model} and heterogeneous model merging. Recent advances introduce more structured approaches, such as Fisher-Weighted Averaging~\cite{singh2021model}, which incorporates Fisher information to weight parameters more effectively, and Permutation Alignment methods like Git Re-Basin~\cite{ainsworth2022git}, which address weight permutation symmetries. Interference Resolution techniques, including TIES~\cite{liu2023resolving} and \texttt{DOGE}~\cite{wei2025modeling}, mitigate parameter conflicts either through explicit alignment or projective gradient descent. Task Arithmetic~\cite{mitchell2022editing} enables weight-space operations to combine task-specific behaviors in language models, while Diversity-Aware Merging, such as DARE~\cite{liu2024dare}, leverages model diversity to improve sparse-to-dense integration. In contrast to the data-free methods mentioned above, data-informed methods~\cite{yang2023adamerging, yang2024representation, tang2023concrete} optimize merging coefficients using additional data. Model merging is impactful for LLMs, enabling efficient knowledge integration without full retraining, facilitating distributed fine-tuning~\cite{wang2022lofi}, multi-task learning~\cite{prakash2023unival}, and cost-effective model adaptation.

\section{Method}

\begin{figure*}
    \centering
    \includegraphics[width=\textwidth]{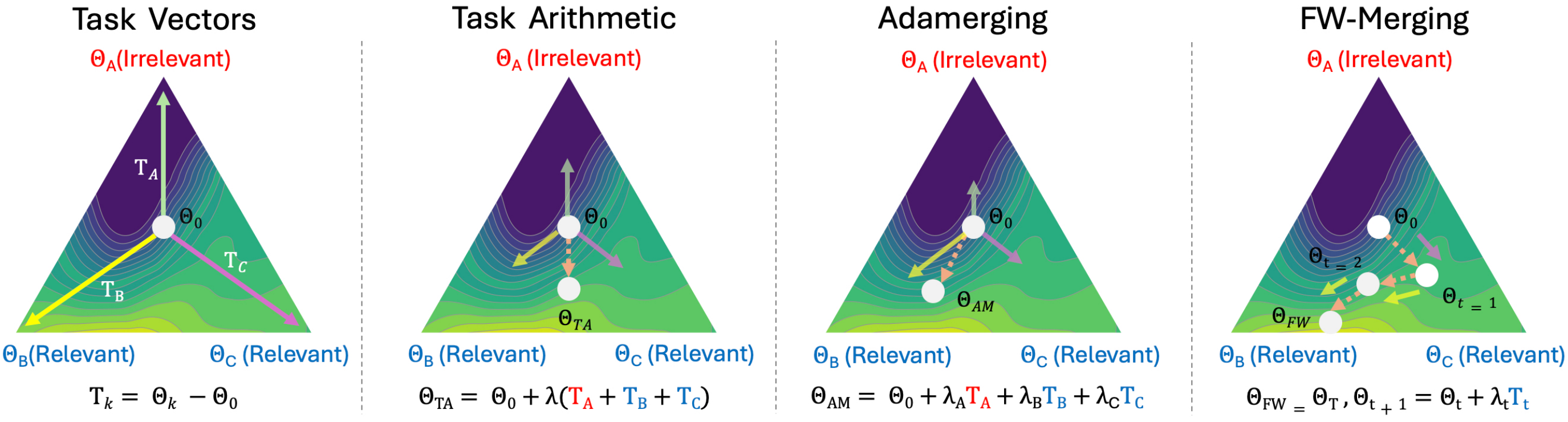}
    \caption{
    Illustration of model merging methods. $\Theta_A$ is an \textcolor{red}{irrelevant} model, while $\Theta_B$ and $\Theta_C$ are \textcolor{blue}{relevant} models. Darker regions indicate higher objective function loss. Task Arithmetic treats all task vectors equally, failing to move optimally. Adamerging assigns different coefficients, moving towards more desirable direction but suffer from slow convergence due to interference from $\Theta_aA$. \fwm iteratively selects the most relevant model to merge and adapts step sizes, efficiently reaching the optimum after $T$ iterations.\vspace{-5mm}}
    \label{fig:overview}
\end{figure*}

\subsection{Preliminary: Frank-Wolfe algorithm}
The Frank-Wolfe (FW) algorithm \cite{frank1956algorithm}, also known as the conditional gradient method, is an iterative optimization algorithm for constrained optimization problems of the form:
\begin{equation}
    \min_{x \in \mathcal{C}} f(x)
\end{equation}
where $f$ is a continuously differentiable function, and $\mathcal{C}$ is a compact convex set. The algorithm follows an elegant geometric intuition: at each iteration $t$, FW first identifies which vertex of $\mathcal{C}$ yields the steepest descent direction and then moves towards this vertex to decrease the value of the objective function. More specifically, FW algorithm:
\begin{enumerate}
    \item Constructs a linear subproblem of the original optimization (a.k.a. \emph{linear minimization oracle}) using first-order Taylor expansion at the point $x_t$: 
    \begin{align}
        {\rm LMO}\!\left(\mathcal{C}, x_t\right) := \argmin_{s \in \mathcal{C}} \langle s, \nabla f(x_t) \rangle
    \end{align}
    \item Finds the vertex $s_t$ of the feasible set $\mathcal{C}$ by picking $s_t \in {\rm LMO}(\mathcal{C}, x_t)$.
    \item Takes a careful step from the current point $x_t$ towards this direction $s_t - x_t$, maintaining feasibility through the convex combination: $x_{t+1} = (1-\gamma_t)x_t + \gamma_t s_t$.
\end{enumerate}

The step size $\gamma_t \in [0,1]$ can be chosen by line search 
\begin{align}
    \gamma_t = \argmin_{\gamma \in [0, 1]} f\Big((1-\gamma)x_t + \gamma s_t \Big),
\end{align}
which ensures a sufficient decrease in $f(.)$ at each FW step.


To determine when to stop, the \emph{FW gap} is used to measure the suboptimality in terms of the proximity to the best solution of LMO:
\begin{align}
    g_t := \max_{s \in \mathcal{C}} \langle -\nabla f(x_t), s_t - x_t \rangle,
\end{align}
which is non-negative by definition.

\subsection{Frank-Wolfe Model Merging}
We consider the problem of fine-tuning a pre-trained foundation model on new tasks. 
Given a pre-trained model $\theta_0$ and previously fine-tuned models $\{\theta_1^*, \theta_2^*, \cdots, \theta_n^*\}$, 
we aim to fine-tune the $(n+1)$-th model $\theta_{n+1}$ on new tasks as a convex combination of the previous models with optimal merging coefficients. 

To this end, we propose a Frank-Wolfe based model merging framework, which is described as follows. 
\begin{align}
    \label{eq:fw_orig}
    \min_\lambda \ell\Big( \sum_{i=1}^n \lambda_i \theta_i^* \Big) \quad \text{s.t.} \quad \sum_{i=1}^n \lambda_i = 1, \quad \lambda_i \geq 0,
\end{align}
where $\lambda = \{\lambda_1, \ldots, \lambda_n\}$ are the merging coefficients, and $\ell$ is a loss function formulated for a specific goal, such as aligning the merged and individual models~\cite{wei2025modeling} or satisfying a task objective~\cite{yang2023adamerging}. 

Potentially, a scaling issue of this formulation appears when the number of fine-tuned models $n$ is large, since we need to keep all the fine-tuned models in memory. 
To address this, we propose a reformulation of the problem eq.~\eqref{eq:fw_orig} as follows: 
\begin{align}
    \label{eq:fw_conv}
    \min_{\theta \in \mathcal{M}} \ell( \theta),
\end{align}
where $\mathcal{M} := {\rm conv}(\{\theta_i^*\}_{i=1}^n)$ is the convex hull of the set of previously fine-tuned models. 

\begin{proposition}
The optimization problems in equations \eqref{eq:fw_orig} and \eqref{eq:fw_conv} are equivalent.
\end{proposition}

\begin{proof}
By definition of convex hull, any point $\theta \in \mathcal{M}$ can be written as a convex combination of the vertices $\{\theta_i^*\}_{i=1}^n$, 
i.e., $\theta = \sum_{i=1}^n \lambda_i \theta_i^*$ where $\lambda \in \Delta^n := \{ \lambda \in \mathbb{R}^n \mid \sum_{i=1}^n \lambda_i = 1, \lambda_i \geq 0 \}$. Therefore:
\begin{align*}
    \min_{\theta \in \mathcal{M}} \ell(\theta) &= \min_{\lambda \in \Delta^n} \ell\Big(\sum_{i=1}^n \lambda_i \theta_i^*\Big).
\end{align*}
This shows that any solution of one problem can be mapped to a solution of the other problem with the same objective value.
\end{proof}

Since the FW algorithm requires the initial solution to be an interior point of the constraint set, 
we add the initial solution $\theta_0$ to form a new constraint set $\mathcal{M} := {\rm conv}(\mathcal{M} \cup \{\theta_0\})$,
which we still denote as $\mathcal{M}$ for simple notations. 
A nice property of this reformulation is that the LMO can be simplified to 
\begin{align}
    {\rm LMO}\!\left(\{\theta_i^*\}_{i=1}^n, \theta_t\right) = \argmin_{s \in \rm \{ \theta_1^*, \ldots, \theta_n^* \}} \langle \nabla \ell(\theta_t), s \rangle
\end{align}

This is because for linear programming problems over convex sets, the optimal solution is always attained at the vertices of the constraint set. 
Algorithm \ref{alg:FW-MM} details the steps, and Figure~\ref{fig:overview} provides an overview.

\begin{algorithm}
	\caption{Frank-Wolfe Merging:}
	\label{alg:FW-MM}
    \KwIn{Initial solution $\theta_0$; Fine-tuned checkpoints $\{\theta_i^*\}^n_{i=1}$; FW budget $T$.}
    \KwOut{Merged model $\theta^*_{merged}$.}
    \begin{algorithmic}[1]
    \STATE \textbf{if} $\theta_0 \notin \mathcal{M}$ \textbf{then} $\mathcal{M} := \mathcal{M} \cup \{\theta_0\}$
	\FOR{$t=0\dots T$}
		\STATE Let $s_t := {\rm LMO}\!\left( \theta_t \right)$ %
		   and $d_t := s_t - \theta_t$ 
            \IF{$g_t := \langle -\nabla \ell(\theta_t), d_t \rangle \leq \epsilon$}  
            \STATE \textbf{return} $\theta^*_{\text{merged}} \gets \theta_t$ \\
            \ENDIF
		  \STATE Line-search: $\gamma_t \in \displaystyle\argmin_{\gamma \in [0,1]} \textstyle \ell\left(\theta_t + \gamma d_t)\right)$ 
		  \STATE Update: $\theta_{t+1} := {\rm MergeFn}(\theta_t, s_t, \gamma_t)$  
	\ENDFOR
    \RETURN $\theta^*_{Merged} \leftarrow \theta_T$
	\end{algorithmic}
\end{algorithm}

\subsection{Design choices of the algorithm} ~\label{method:design}
The above algorithm illustrates the key ingredients of Frank-Wolfe merging: ${\rm LMO}$, stopping criterion $g_t$, line search routine, and the merging function. We discuss in this section the design choices of these components.  

\paragraph{Merging function}
The main deviation from the classical FW algorithm is the reinterpretation of the FW update $x_{t+1} = (1-\gamma_t)x_t + \gamma_t s_t$ as a local merging between $\theta_t$ and $s_t$. We denote by ${\rm MergeFn}$ the customizable merging function as long as the merged model stays in the convex hull $\mathcal{M}$.
The most straightforward merging function is the convex combination:
\begin{align}
    {\rm MergeFn}(\theta_t, s_t, \gamma_t) := (1-\gamma_t) \theta_t + \gamma_t s_t,
\end{align}
which corresponds to the Task-Arithmetic \cite{mitchell2022editing} method. 
The step size $\gamma_t$ makes sure the merged model stays in $\mathcal{M}$.


It is natural to ask whether other existing model merging methods, such as TIES-Merging \cite{liu2023resolving} and DARE \cite{liu2024dare} could also be used as ${\rm MergeFn}$. The problem with these sophisticated merging methods is that the merged model might leave the constraint set, and thus violate the assumption of maintaining feasibility required by the classical FW theory. We verified in practice that these less rigorous merging functions might achieve better performance in certain cases but they generally cause more stability issues. Therefore, we do not consider these merging functions from the current comparison.

\paragraph{Hard FW v.s. Soft FW}
In the case of deep learning, the optimization problem is non-convex, additional efforts are needed to better characterize the loss landscape and prevent the LMO from being dominated by one or a few fine-tuned models, which occurs because the linear approximation of $\ell(\theta)$ is an inaccurate sketch of the original objective function. 
Instead of relying on the \emph{argmin} of linear subproblem, we fetch the top-$k$ vertices of LMO, $\{ \tilde{s}_j \}_{j=1}^k$. A more subtle top-$k$ operation can be performed in a task-wise fashion if the original objective function involves multi-tasks.

Given the top-$k$ vertices, we now go back to eq.~\eqref{eq:fw_orig} to obtain the optimal merging coefficients $\{ \lambda_j^* \}_{j=1}^k$. Note that this inner optimization\footnote{For inner optimization, we use projected gradient descent with a projection of $\{ \lambda_j \}_{j=1}^k$ onto the simplex after each gradient update.} is a reduced version of original eq.~\eqref{eq:fw_orig} because hosting $k$ models in memory would not be a problem. 
We also remove the line search step as this gives a new merging function of the form
\begin{align}
    \label{eq:soft_lmo}
    {\rm MergeFn}(\theta_t, \{ \tilde{s}_j \}_{j=1}^k, \{ \lambda_j^* \}_{j=1}^k)
    := \theta_t + \sum_{j=1}^k \lambda_j^* (\tilde{s}_j - \theta_t).
\end{align}
We call this oracle \emph{soft} LMO in comparison to the argmin version which we call \emph{hard} LMO. 

\begin{proposition}
The merging function maintains feasibility, i.e., the merged model stays in the convex hull $\mathcal{M}$.
\end{proposition}

\begin{proof}
    We can rewrite the merging function as:
    \begin{align*}
        \theta_{t+1} = \big( 1 - \sum_{j=1}^k \lambda_j^* \big) \cdot \theta_t + \sum_{j=1}^k \lambda_j^* \cdot \tilde{s}_j.
    \end{align*}
    Since $\theta_t \in \mathcal{M}$ and $\tilde{s}_j \in \mathcal{M}$ for all $j=1,\dots,k$, 
    and $\{\lambda_j^*\}_{j=1}^k$ are obtained through projection onto the simplex 
    (i.e., $\sum_{j=1}^k \lambda_j^* = 1$ and $\lambda_j^* \geq 0$), 
    we have $\theta_{t+1} \in \mathcal{M}$.
    This follows from the convexity of $\mathcal{M}$: a convex combination of points in a convex set remains in the set.
\end{proof}

\begin{theorem}[Convergence Rate of Soft FW]
    Consider $\ell(\theta)$ be $L$-smooth over $\mathcal{M}$, which has two constants:
    ${\rm diam} := \max_{\theta_1, \theta_2 \in \mathcal{M}} \|\theta_1 - \theta_2\|$ be the diameter of $\mathcal{M}$,
    and ${\rm subopt} := \ell(\theta_0) - \min_{\theta \in \mathcal{M}} \ell(\theta)$ be the global suboptimality.
    Consider the soft FW algorithm which introduces the following changes to Algorithm~\ref{alg:FW-MM}:
    \begin{enumerate}
        \item $\{ \tilde{s}_j \}_{j=1}^k$ is the top-$k$ vertices of LMO.
        \item $\{\lambda_j^*\}_{j=1}^k = \arg\min_{\lambda \in \Delta^k} \ell(\theta_t + \sum_{j=1}^k \lambda_j (\tilde{s}_j - \theta_t))$.
        \item $\theta_{t+1} = \theta_t + \sum_{j=1}^k \lambda_j^* (\tilde{s}_j - \theta_t)$.
    \end{enumerate}
    We have:
    \begin{align*}
        \min_{t = 0, \ldots, T} g_t \leq \frac{\rm subopt}{T} + \frac{L \cdot {\rm diam}^2}{2}.
    \end{align*}
\end{theorem}
    
\begin{proof}
    We first define $g_t^k$ as the top-$k$ FW gap of the soft FW algorithm:
    \begin{align*}
        g_t^k := \max_{\lambda \in \Delta^k} \max_{s_1, \ldots, s_k \in \mathcal{M}} 
        \sum_{j=1}^k \lambda_j \langle \nabla \ell(\theta_t), \theta_t - s_j \rangle.
    \end{align*}
    Comparing to the full FW gap
    \begin{align*}
        g_t = \max_{s \in \mathcal{M}} \langle \nabla \ell(\theta_t), \theta_t - s \rangle,
    \end{align*}
    we have:
    \begin{align*}
        g_t^k \geq g_t
    \end{align*}
    because the top-$k$ FW gap subsumes the original FW gap by setting $\lambda_1 = 1$ and $\lambda_j = 0$ for $j=2,\dots,k$.
    Intuitively, selecting multiple descent directions and optimizing their combination always gives at least as much descent as the single best direction.
    From the Lipschitz continuity of $\ell(\theta)$, we have:
    \begin{align*}
        \ell(\theta_{t+1}) \leq \ell(\theta_t) + \langle \nabla \ell(\theta_t), \theta_{t+1} - \theta_t \rangle + \frac{L}{2} \|\theta_{t+1} - \theta_t\|^2.
    \end{align*}
    Using the update rule $\theta_{t+1} = \theta_t + \sum_{j=1}^k \lambda_j^* (\tilde{s}_j - \theta_t)$, we have:
    \begin{align*}
        \langle \nabla \ell(\theta_t), \theta_{t+1} - \theta_t \rangle = -g_t^k.
    \end{align*}
    Therefore, 
    \begin{align*}
        \ell(\theta_{t+1}) \leq \ell(\theta_t) - g_t^k + \frac{L}{2} \|\theta_{t+1} - \theta_t\|^2.
    \end{align*}
    Since $\theta_{t+1}$ is a convex combination of $\theta_t$ and $\tilde{s}_j$, we have:
    \begin{align*}
        \| \theta_{t+1} - \theta_t \|^2 \leq {\rm diam}^2.
    \end{align*}
    Hence, 
    \begin{align*}
        \ell(\theta_{t+1}) \leq \ell(\theta_t) - g_t^k + \frac{L}{2} {\rm diam}^2.
    \end{align*}
    Summing over $t=0,\dots,T-1$, we have:
    \begin{align*}
        \sum_{t=0}^{T-1} g_t^k &\leq \ell(\theta_0) - \ell(\theta_T) + \frac{LT}{2} {\rm diam}^2. \\
        &\leq {\rm subopt} + \frac{LT}{2} {\rm diam}^2.
    \end{align*}
    Therefore,
    \begin{align*}
        \min_{t=0,\dots,T} g_t^k \leq \frac{1}{T} \sum_{t=0}^{T-1} g_t^k \leq \frac{{\rm subopt}}{T} + \frac{L}{2} {\rm diam}^2.
    \end{align*}
    The same result holds for $g_t$ by the definition of $g_t^k$.
\end{proof}

This convergence proof for non-convex objective functions is based on the proof given by \cite{lacoste2016convergence}. Due to the soft LMO, we obtain a better convergence rate $O(\frac{1}{T})$ over the vanilla rate $O(\frac{1}{\sqrt{T}})$ with a price to solve a relatively more expensive iteration to obtain the optimal coefficients. This might result in a longer total time, but it is worthy of a solution to the problem of model merging. 

\paragraph{Task-wise LMO v.s. layer-wise LMO}
The naive implementation of \fwm would vectorize the whole model weights $\theta$ and then solve LMO. We call this \emph{task-wise} LMO. Since different layers contribute differently to model performance~\cite{yosinski2014transferable}, a \emph{layer-wise} LMO may yield better model merging. To incorporate this, the constraint set is redefined as a Cartesian product of convex hulls for each layer: $\mathcal{M} := \mathcal{M}_1 \times \cdots \times \mathcal{M}_L$, 
where $L$ is the number of layers and $\mathcal{M}_l := \text{conv}\left(\{\theta_i^{*, l}\}_{i=1}^n\right)$. The LMO is then conducted layer-wise:  
\begin{align}    
{\rm LMO}(\{\theta_i^{*, l}\}_{i=1}^n, D, \theta_t^l) = \argmin_{s^l \in \{\theta_1^{*,l}, \dots, \theta_n^{*,l}\}} \langle \nabla \ell(\theta^t)^l, s^l \rangle.
\end{align}  
This version can be viewed as a block-coordinate Frank-Wolfe algorithm \cite{lacoste2013block}, which is applied when the problem has a natural decomposition into blocks.

\begin{table*}[h]
    \centering
    \caption{Performance on 4 Discriminative Tasks when merging 8 RoBERTa and 3 Generative Tasks when merging 16 LLaMA2-7B.
    }
    \scalebox{0.95}{
        \begin{tabular}{lccc}
    \toprule
    \textbf{Method} & \textbf{4 Disc. Tasks (8 Models)} & \textbf{3 Gen. Tasks (16 Models)} & \textbf{Avg. Normalized Score} \\
    \midrule
    \rowcolor[gray]{0.9} Pretrained & 49.6 & 77.1 & 63.4 \\
    \rowcolor[gray]{0.9} Traditional MTL & 73.1 & 81.2 & 77.2 \\
    \midrule
    Task Arithmetic (w/ DARE) & 77.3 & 16.8 & 47.1 \\
    Ties-Merging (w/ DARE) & 75.6 & 46.6 & 61.1 \\
    Task Arithmetic & 80.8 & 75.9 & 78.4 \\
    Ties-Merging & 64.3 & 78.5 & 71.4 \\
    \midrule
    \rowcolor{lightblue} \fwta (Ours) & \textbf{85.4} & \textbf{81.1} & \textbf{83.1} \\
    \bottomrule
\end{tabular}

    }
    \label{tab:main_nlp}
\end{table*}

\begin{table}[h]
    \centering
    \caption{Performance on 4 CV Tasks when merging 20 ViT-B/32.
    }
    \scalebox{0.85}{
        \begin{tabular}{lcccccc}
    \toprule
    \textbf{Method} & \textbf{SUN397} & \textbf{Cars} & \textbf{GTSRB} & \textbf{DTD} & \textbf{Avg.} \\
    \midrule
    \rowcolor[gray]{0.9} Pretrained & 62.3 & 59.7 & 32.6 & 43.8 & 49.6  \\
    \midrule
    DARE (TIES) & 5.9 & 2.3 & 16.7 & 11.8 & 9.2  \\
    Task Arithmetic & 20.4 & 12.2 & 29.8 & 22.3 & 21.2 \\
    Ties-Merging & 51.0 & 36.2 & 57.7 & 40.6 & 46.4 \\
    Weight Averaging & 64.2 & 59.6 & 43.1 & 46.5 & 53.4  \\
    Fisher Merging & 64.6 & 63.8 & 43.0 & 46.9 & 54.6 \\
    RegMean & 65.5 & 62.2 & 59.7 & 53.9 & 60.3 \\
    \midrule
    LW Concrete AM & 62.5 & 60.3 & 88.0 & 54.7 & 66.3  \\
    Adamerging & 66.4 & 70.1 & 95.1 & 64.0 & 73.9  \\
    Surgery & 69.7 & 71.8 & 96.6 & 73.4 & 77.9 \\
    \rowcolor{lightblue} \fwta (Ours) & 66.5 & 69.9 & 95.1 & 64.5 & 74.0 \\
    \rowcolor{lightblue} \fwam (Ours) & \textbf{72.9} & \textbf{74.8}  & \textbf{96.8} & \textbf{76.0} & \textbf{80.1} \\
    \bottomrule
\end{tabular}

    }
    \label{tab:main_cv}
\end{table}

\section{Experiments}

\subsection{Experiment Setup}

\paragraph{Benchmarks.} 
Our primary objective is to evaluate the effectiveness of our method in scenarios where the number of models greatly exceeds the number of evaluation tasks, and each model’s capabilities are unknown in advance.  

\begin{itemize}  
    \item \textbf{Vision Tasks:} Following the setting of TALL~\cite{wang2024localizing}, we use \textbf{20} ViT-B/32 models, each fine-tuned on a different vision task. The number of models to be merged is intentionally set to be significantly larger than the number of evaluation tasks, allowing us to assess the scalability of model merging methods. The evaluation benchmarks consist of four tasks: SUN397~\cite{xiao2016sun}, Stanford Cars~\cite{cars}, GTSRB~\cite{gtsrb}, and DTD~\cite{dtd}.
    \item \textbf{Language Discriminative Tasks:} We prepare \textbf{8} RoBERTa checkpoints~\cite{liu2019roberta} fine-tuned on eight tasks from the GLUE benchmark, following the practice in~\cite{lu2025twin}. The merged model is then evaluated on four tasks from the GLUE benchmark~\cite{wang2018glue}: MNLI, QNLI, QQP, and RTE.  
    \item \textbf{Language Generative Tasks:} We collect \textbf{16} LLaMA2-7B models~\cite{touvron2023llama} fine-tuned with LoRA~\cite{hu2022lora} on various tasks from Hugging Face. These models have unknown and uncontrolled capabilities, making them equivalently black-box models. Our goal is to evaluate the robustness of model merging methods in this challenging setting. The evaluation benchmarks include CNN/DM summarization~\cite{cnn}, PubMedQA~\cite{jin2019pubmedqa}, and HumanEval~\cite{chen2021codex}.
\end{itemize}  

Detailed information can be found in \textbf{Appendix}~\ref{app:benchmark}.

\paragraph{Metrics.}  
For vision tasks, we report the classification accuracy. Following~\cite{lu2025twin}, we report the average normalized score for language tasks to account for differences in task-specific score ranges to account for variations in task-specific score ranges. The normalized score is computed as
$
\text{Score}_{\text{normalized}} = \frac{1}{T} \sum_{t=1}^{T} \frac{\text{Score} \left( f(\boldsymbol{\theta}^*) \right)}{\text{Score} \left( f_t(\boldsymbol{\theta}_t) \right)}  
$.

\paragraph{Baselines.}  
We compare \fwm with both data-informed and data-free model merging methods.  
For data-informed model merging, we compare \fwm against Adamerging~\cite{yang2023adamerging}, Surgery~\cite{yang2024representation}, and Concrete Merging~\cite{tang2023concrete}. To ensure a fair comparison, these methods are trained only on the same tasks as \fwm.  
For data-free model merging, we compare \fwm with Fisher Merging, Weight Averaging, RegMean Merging, Task Arithmetic~\cite{mitchell2022editing}, Ties-Merging~\cite{yadav2023ties}, and DARE Merging~\cite{liu2024dare} across both language and vision tasks.  
Additionally, we fine-tune one model on the discriminative language benchmark and another on the generative language benchmark to serve as additional baselines.  
Further details can be found in Appendix~\ref{app:baselines}.  

\paragraph{Implementations.}
We implement two \fwm variants: \fwta, which uses hard FW with layer-wise LMO, and \fwam, which employs soft FW with task-wise LMO (Section~\ref{method:design}).
For \fwam, layer-wise coefficients are optimized via gradient descent on the training dataset to solve eq.~\ref{eq:fw_orig}, differing from Adamerging~\cite{yang2023adamerging} by minimizing cross-entropy loss on training data rather than entropy on test samples. 
On language benchmarks, the training dataset consists of 100 samples per task, and \fwta runs for 10 iterations, initialized with Task Arithmetic’s merged model. For vision tasks, it runs for 3 iterations. \fwam runs for 15 iterations on vision benchmarks, initialized with the pre-trained model. Training datasets consist of 100 samples from MNLI, QNLI, QQP, and RTE~\cite{wang2018glue} for discriminative tasks; CNN/DM~\cite{cnn}, CodeAlpaca-20k~\cite{codealpaca}, and PubMedQA~\cite{jin2019pubmedqa} for generative tasks; and SUN397~\cite{xiao2016sun}, Stanford Cars~\cite{cars}, GTSRB~\cite{gtsrb}, and DTD~\cite{dtd} for vision tasks.
Further details can be found in Appendix~\ref{app:impl}.

\subsection{Comparison with Model Merging Methods}

We evaluate \fwm against both data-informed and data-free model merging approaches across language and vision benchmarks. Table~\ref{tab:main_nlp} reports the results for language tasks, including both discriminative and generative settings, while Table~\ref{tab:main_cv} presents results on vision benchmarks. 

\paragraph{Language Tasks.}  
\fwta achieves the highest average normalized score across language benchmarks, consistently surpassing prior model merging baselines. Specifically, \fwta improves upon Task Arithmetic by 4.6 points, Ties-Merging by 11.7, and DARE (Ties) by 9.8 on discriminative tasks. Table~\ref{tab:cost_finetuned} shows that \fwta also outperforms data-informed Adamerging by 5.9 points.   
For generative tasks, \fwta outperforms Task Arithmetic by 5.2 points, Ties-Merging by 2.6, and DARE (Ties) by 34.5. 
Interestingly, while Task Arithmetic outperforms Ties-Merging on discriminative tasks by a margin of 16.5 points, it lags behind by 2.6 points on the more challenging generative tasks. This discrepancy likely arises from increased interference among task vectors as more checkpoints are merged. Unlike Ties-Merging, which explicitly resolves merging conflicts, Task Arithmetic lacks a reconciliation mechanism, making it more susceptible to such interference.  
In contrast, \fwta consistently outperforms both Ties-Merging and Task Arithmetic by selectively merging only the most relevant model parameters in each iteration. This targeted approach effectively mitigates interference, leading to more stable and robust performance across both discriminative and generative tasks.

\paragraph{Vision Tasks.}  
\fwam achieves state-of-the-art performance across multiple vision benchmarks, surpassing data-informed methods like Adamerging and Surgery. As shown in Table~\ref{tab:main_cv}, \fwta surpasses Adamerging, Concrete Merging, and all data-free merging methods in overall performance. Additionally, \fwam attains the highest accuracy (80.1\%), outperforming Adamerging by 6.2\% and Surgery by 2.2\%. Unlike Surgery, which requires additional task-specific parameters and multiple forward passes per inference, our approach efficiently adapts to diverse visual tasks without increasing storage or inference complexity. 

In general, data-free merging methods show significantly lower performance compared to data-informed approaches while merging a large number of models, when the models' capabilities do not precisely align with the evaluation tasks. This limitation arises because data-free methods treat all models equally, merging them without considering their unique capabilities, which amplifies interference between models. In contrast, data-informed merging methods achieve superior performance by optimizing merging coefficients on datasets as they allow for explicit control over desirable capabilities. \fwm, in particular, enhances scalability via hard model selection based on the linear approximation minimization. 

\begin{table*}
    \centering
    \caption{Merging Methods' Performance vs. Number of Models when Adding Relevant vs. Irrelevant Models.}
    \scalebox{0.99}{
        \begin{tabular}{c|ccccc|ccccc}
    \toprule
    \multirow{3}{*}{\textbf{\#Models}} & \multicolumn{5}{c|}{\textbf{4 CV Tasks}} & \multicolumn{5}{c}{\textbf{20 CV Tasks}} \\
     & \multicolumn{5}{c|}{When "Irrelevant" Models Added} & \multicolumn{5}{c}{When "Relevant" Models Added} \\
    \cmidrule(lr){2-6} \cmidrule(lr){7-11}
    & \textbf{DARE} & \textbf{Task} & \textbf{Ties} & \textbf{AM} & \textbf{\fwam} & \textbf{DARE} & \textbf{Task} & \textbf{Ties} & \textbf{AM} & \textbf{\fwam} \\
    \midrule
    4  & 73.6  & 70.3  & 65.7  & \textbf{75.2}  & \cellcolor{lightblue}74.1  & 57.3  & 59.2  & \textbf{60.2}  & 59.6  & \cellcolor{lightblue}59.2  \\
    6  & 64.1  & 64.1  & 64.4  & \textbf{74.6}  & \cellcolor{lightblue}73.9  & 54.0  & 58.4  & 61.6  & \textbf{64.0}  & \cellcolor{lightblue}63.2  \\
    8  & 48.5  & 57.1  & 63.6  & 73.4  & \cellcolor{lightblue}\textbf{74.0}  & 48.2  & 57.2  & 63.9  & \textbf{64.7}  & \cellcolor{lightblue}64.5  \\
    10  & 40.6  & 52.9  & 62.9  & 72.8  & \cellcolor{lightblue}\textbf{73.9}  & 43.5  & 55.0  & 63.9  & 65.4  & \cellcolor{lightblue}\textbf{66.2}  \\
    12  & 31.7  & 47.9  & 61.9  & 71.4  & \cellcolor{lightblue}\textbf{74.1}  & 39.1  & 52.3  & 63.7  & 65.1  & \cellcolor{lightblue}\textbf{67.5}  \\
    14 & 23.6  & 40.1  & 57.7  & 69.9  & \cellcolor{lightblue}\textbf{74.0}  & 36.1  & 48.9  & 62.0  & 64.1  & \cellcolor{lightblue}\textbf{68.0}  \\
    16 & 17.9  & 33.3  & 52.5  & 68.6  & \cellcolor{lightblue}\textbf{74.1}  & 32.3  & 45.3  & 59.7  & 62.7  & \cellcolor{lightblue}\textbf{68.3}  \\
    18 & 12.2  & 25.9  & 49.2  & 64.6  & \cellcolor{lightblue}\textbf{74.0}  & 27.8  & 40.3  & 57.3  & 60.7  & \cellcolor{lightblue}\textbf{68.4}  \\
    20 &  9.2  & 21.2  & 46.4  & 61.0  & \cellcolor{lightblue}\textbf{74.2}  & 24.2  & 36.3  & 55.6  & 58.9  & \cellcolor{lightblue}\textbf{68.3}  \\
    \bottomrule
\end{tabular}
    }
    \label{tab:scaling}
\end{table*}

\begin{figure*}
    \centering
    \includegraphics[width=\textwidth]{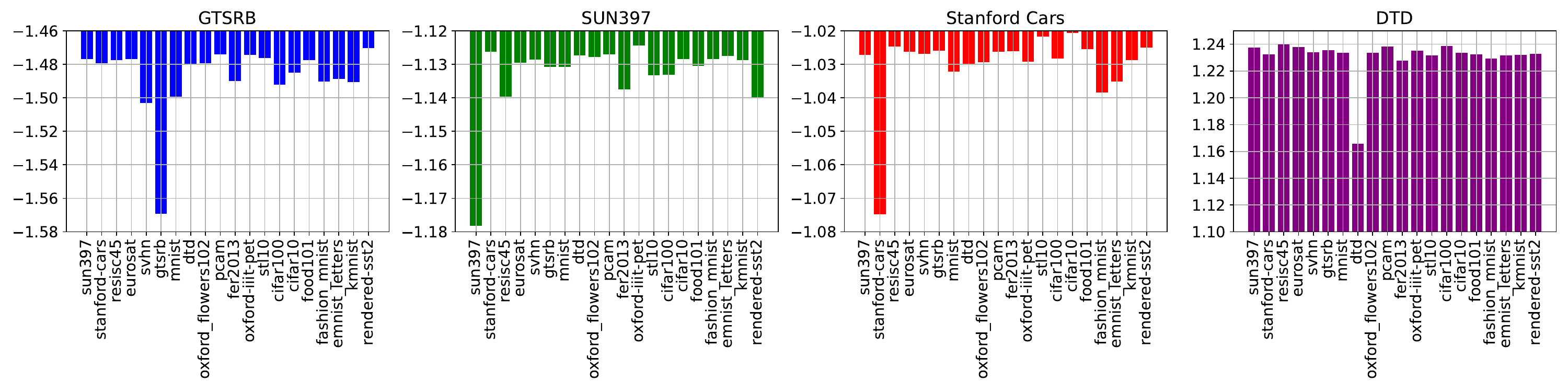}
    \caption{Linear Approximation of the Objective Function of Model Checkpoints Across Different Tasks in a Frank-Wolfe Iteration. The x-axis represents the checkpoints, and each graph shows the linear approximation result for each task.
    }
    \label{fig:inner_products}
\end{figure*}

\subsection{Comparison with Traditional MTL}

\begin{table}
    \centering
    \caption{Costs and Perf. of methods on NLP discriminative tasks.
    }
    \scalebox{0.8}{
        \begin{tabular}{lccc}
    \toprule
    \textbf{Method}  & \textbf{Data Samples/Task} & \textbf{Time Cost} & \textbf{Perf.} \\
    \midrule
    Traditional MTL & 2.9K & 4.2h & 73.1 \\
    Data-free Merging & 0 & 0 & 80.8 \\
    Data-informed Merging & 1.6K & 2min & 79.5 \\
    \rowcolor{lightblue} FW Merging & 100 & 2min & 85.4 \\
    \bottomrule
\end{tabular}

    }
    \label{tab:cost_finetuned}
\end{table}

We compare \fwm with models fine-tuned using traditional MTL on discriminative and generative tasks. In each case, one single model is fine-tuned across all tasks, with performance and computational cost reported in Table~\ref{tab:main_nlp} and Table~\ref{tab:cost_finetuned}. 
Traditional MTL achieves an average score of 77.2, lower than that of \fwm (83.1). On discriminative tasks, MTL scores 73.1, trailing \fwm (85.4). For generative tasks, MTL scores 81.2, while \fwm closely matches it at 80.8, suggesting that \fwm's performance matches that of traditional MTL.

As shown in Table~\ref{tab:cost_finetuned}, \fwm demonstrates a substantial advantage in efficiency. Traditional MTL requires fine-tuning on 2.9K samples per task and takes 4.2 hours of training time, which is computationally intensive. In contrast, \fwm only requires 100 training samples per task and completes the merging process in just 2 minutes. This huge reduction in computational cost underscores the effectiveness of \fwm compared to traditional MTL. Moreover, \fwm has a key advantage over traditional MTL: while MTL requires a large volume of high-quality data for optimal performance, \fwm needs only a small set of data because: \textit{1)} it optimizes merging coefficients based on models' characteristics, which simplifies the optimization space, and \textit{2)} it uses model weights as inputs, which are much more information-dense representations than data, enabling more efficient objective learning.

\fwm is a post-training technique that does not require access to original training data, making it ideal for privacy-sensitive or data-scarce scenarios. Overall, the results suggest that \fwm is a scalable, efficient alternative to traditional MTL, providing comparable performance at a reduced computational cost.

\subsection{Scaling to More Models and Tasks}~\label{exp:scaling}

We investigate the performance scaling of different merging methods with the number of models, as shown in Figure~\ref{fig:scaling1}, Figure~\ref{fig:scaling2}, and Table~\ref{tab:scaling}. In large-scale model merging, models from open-source platforms vary in quality. To simulate this, we use 20 ViT-B/32 models fine-tuned on tasks that are either included in the evaluation benchmark or not. A model is \textit{irrelevant} if its fine-tuning dataset does not match the training split of the evaluation task, and \textit{relevant} if it matches. To ensure fair comparison, the total number of training iterations run by \fwam is the same as that of Adamreging.

As shown in Table~\ref{tab:scaling}, adding \textit{irrelevant} models sharply reduces the performance of data-free methods: DARE by 64.4\%, Task Arithmetic by 49.1\%, and Ties by 19.1\%, likely due to task interference and equal treatment of all models. Data-informed methods degrade less, with Adamerging dropping by 14.2\%. In contrast, \fwam remains highly stable, fluctuating only from 73.9\% to 74.1\% as more models are added.
In Figure~\ref{fig:inner_products}, we examine the linear approximation of different checkpoints for a specific task and find that the model fine-tuned on the task consistently yields the most negative linear approximation. This indicates that in the Frank-Wolfe update, the most relevant checkpoint is chosen as the direction for merging, allowing \fwm to iteratively improve the merged model in the optimized direction within the constraint set. The inner product between gradients and model parameters serves as a reliable indicator of model relevance, with minized computational cost, further demonstrating \fwm’s scalability even in the presence of irrelevant models.

Adding \textit{relevant} models should ideally improve performance, but data-free methods still degrade as shown in Table~\ref{tab:scaling}: DARE by 33.1\%, Task Arithmetic by 22.9\%, and Ties by 4.6\%, with Ties performing best by mitigating parameter conflicts. Data-informed methods like Adamerging fluctuate between 58.9\% and 64.7\% as merging complexity increases, whereas \fwam steadily improves from 59.2\% to 68.3\% by iteratively selecting the most relevant models, facilitating smoother convergence.
These results underscore \fwm’s effectiveness as a scalable solution for large-scale model merging.

\subsection{Ablation Studies}


\begin{table}
    \centering
    \caption{Ablation on design variants of \fwm.
    }
    \scalebox{1}{
        \begin{tabular}{cccc}
    \toprule
    \textbf{Coefficient $\lambda$} & \textbf{Method} & \textbf{LMO} & \textbf{Score} \\
    \midrule
    \multicolumn{4}{c}{\textit{Vision Tasks}} \\
    Optimized & \fwam & Layer-wise & 79.7 \\
    Optimized & \fwam & Task-wise & 80.1 \\
    Unoptimized & \fwam & Layer-wise & 69.8 \\
    Unoptimized & \fwam & Task-wise & 70.3 \\
    - & \fwta & Layer-wise & 74.0 \\
    - & \fwta & Task-wise & 73.7 \\
    \midrule
    \multicolumn{4}{c}{\textit{NLP Discriminative Tasks}} \\
    - & \fwta & Layer-wise & 85.4 \\
    - & \fwta & Task-wise & 78.2 \\
    \bottomrule
\end{tabular}

    }
    \label{tab:ablation}
\end{table}

\paragraph{Design variants.} \label{exp:layerwise}
Table~\ref{tab:ablation} compares the design variants of \fwm (Section~\ref{method:design}). Task-wise LMO aligns better with \fwam, improving performance slightly by 0.5 points over layer-wise LMO, while layer-wise LMO is more effective for \fwta, especially on language tasks, yielding a 7.2-point gain. This is likely because \fwam optimizes layer-wise coefficients during merging, reducing the impact of layer-wise selection.

\fwam excels when merging a large number of models, outperforming \fwta by up to 6.7 points. Its ability to select multiple optimal directions per iteration allows it to navigate the parameter space efficiently.

Optimizing merging coefficients $\lambda$ further improves performance by up to 9.9 points, underscoring the importance of weighting model parameters based on their relevance.

\begin{figure}
    \centering
    \begin{subfigure}{0.48\linewidth}
        \centering
        \includegraphics[width=\linewidth]{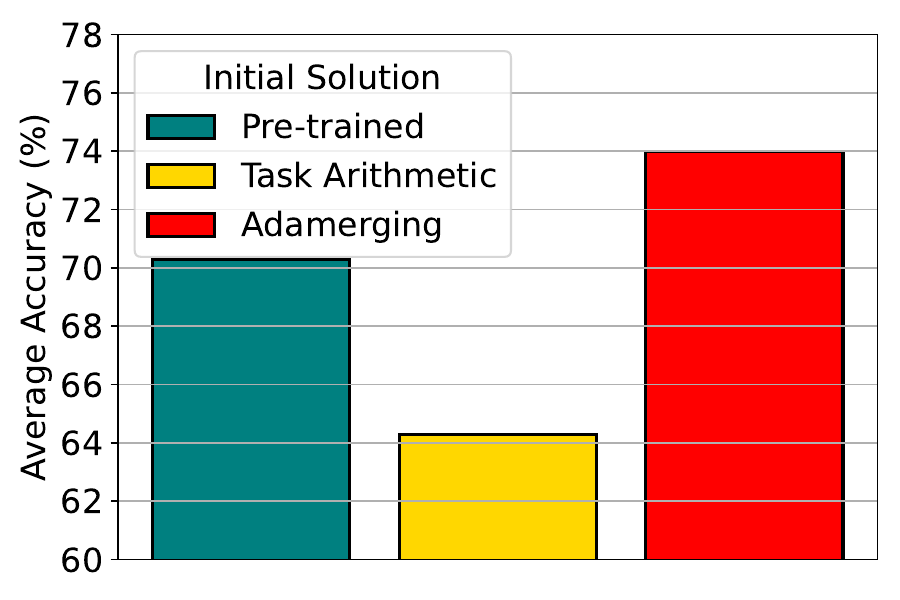}
        \caption{Initial soluiton
        }
        \label{fig:init_point}
    \end{subfigure}
    \hfill
    \begin{subfigure}{0.48\linewidth}
        \centering
        \includegraphics[width=\linewidth]{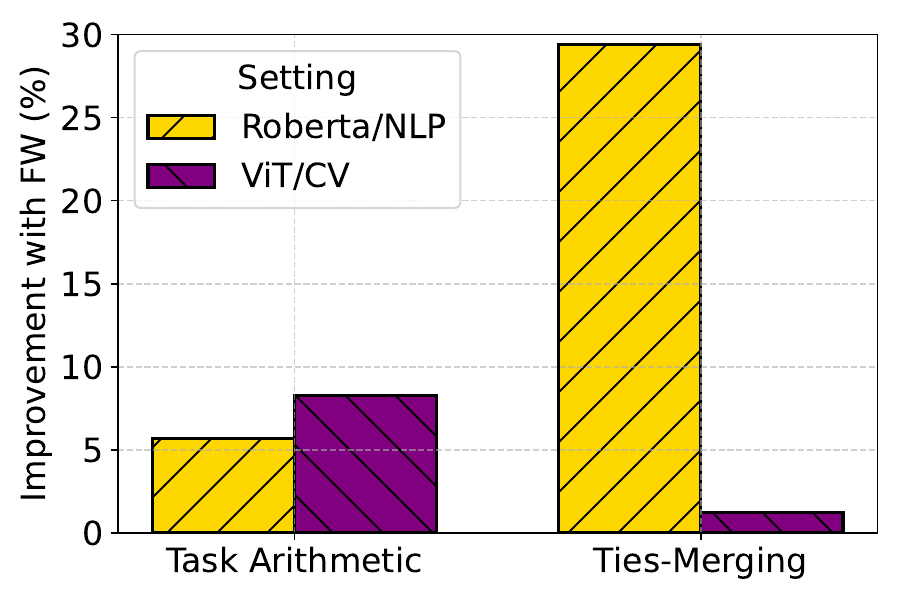}
        \caption{Merging functions
        }
        \label{fig:merge_fn}
    \end{subfigure}
    \caption{Ablation on \fwm. (a) reports accuracies on the vision benchmark, while (b) on vision and language benchmarks.
    }
\end{figure}

\paragraph{Initial solution.}
We examine the effect of initialization on \fwm. An ideal initial solution should either (1) be closer to the global optimum or (2) expand the constraint set with a more meaningful search space. As shown in Figure~\ref{fig:init_point}, initializing \fwm with the Adamerging result improves performance compared to starting from the pre-trained model, likely because Adamerging is closer to the optimal point. In contrast, task arithmetic leads to worse performance than the pre-trained model, potentially due to its poor performance on vision tasks (21.2\%), suggesting it starts further from the optimum. Consequently, more FW iterations are required to achieve convergence.  

\paragraph{Flexibility of merging functions.}
Although only a restricted set of merging functions ensure that \fwm remains within the convex hull, we demonstrate the flexibility of \fwm by showing its ability to enhance alternative merging functions. As shown in Figure~\ref{fig:merge_fn}, applying \fwm with both Task Arithmetic and Ties-Merging improves performance on NLP and vision tasks, even though Ties-Merging does not necessarily stay within the convex hull. This suggests that \fwm remains effective across different merging functions.

\section{Conclusion}
In this work, we extend model merging to a more challenging setting where the merging pool consists of a large number of black-box fine-tuned checkpoints. While existing methods require prior knowledge of model details to achieve optimized performance, our proposed Frank-Wolfe Merging (\fwm) scales effectively with a large number of black-box models, iteratively refining the merged model towards the optimal point defined by an objective function. Experiments demonstrate that \fwm achieves superior performance and scalability, paving the way for next-generation model merging.

\section{Acknowledgement} 

The support of the United Kingdom EPSRC (grant number UKRI256, EP/V028251/1, EP/N031768/1, EP/S030069/1,
and EP/X036006/1), Intel, and AMD is gratefully acknowledged.

{
    \small
    \bibliographystyle{ieeenat_fullname}
    \bibliography{main}
}
\appendix

\section{Data Efficiency}

\begin{figure}[h]
    \centering
    \includegraphics[width=\linewidth]{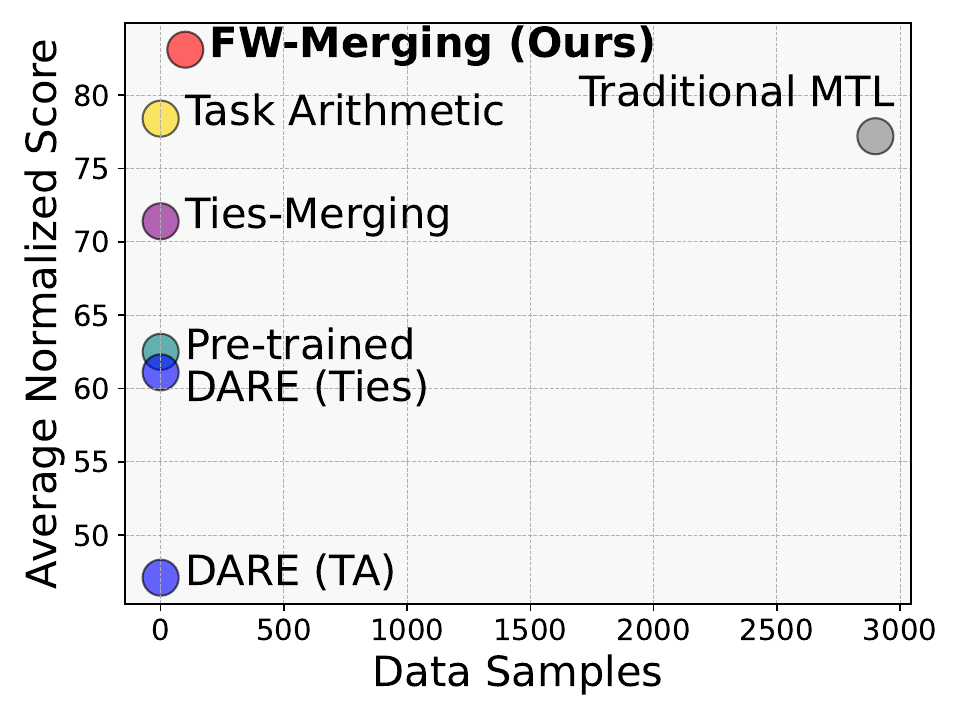}
    \caption{Performance vs. \#Data Samples.}
    \label{fig:scaling_data}
\end{figure}

As illustrated in Figure~\ref{fig:scaling_data}, \fwm outperforms all other model merging methods in terms of performance for the language benchmark. Its performance also surpasses that of traditional MTL while using less training data.

\section{Experiment Details}

\subsection{Benchmarks} \label{app:benchmark}

\paragraph{Discriminative Tasks.} Following previous research~\cite{lu2025twin}, 10\% of the training split is used as validation split, while the original validation set is used as test set. We fine-tuned 8 RoBERTa on 8 tasks form the GLUE benchmark~\cite{wang2018glue}: QNLI, COLA, STS-B, QQP, SST-2, MRPC, MNLI, RTE. For the evaluation benchmark, we use MNLI, QNLI, QQP, and RTE.

\paragraph{Generative Tasks.} We collected the following fine-tuned LLaMA2-7B checkpoints from Hugging Face:

\begin{itemize}
    \item Code Generation\footnote{\url{https://huggingface.co/arnavgrg/codealpaca-qlora}}
    \item Medical QA\footnote{\url{https://huggingface.co/SanjanaR01/medical-dialogue-summary-llama2-7b-peft-qlora}}
    \item News Summarization\footnote{\url{https://huggingface.co/ernlavr/llama2_7bn-xsum-cnn-lora-adapter}}
    \item Commonsense Reasoning\footnote{\url{https://huggingface.co/Styxxxx/llama2_7b_lora-piqa}}
    \item Machine Translation\footnote{\url{https://huggingface.co/Styxxxx/llama2_7b_lora-wmt16_translate_roen}, 
    \url{https://huggingface.co/Styxxxx/llama2_7b_lora-wmt16_translate_csen}, 
    \url{https://huggingface.co/Styxxxx/llama2_7b_lora-wmt16_translate_deen}, 
    \url{https://huggingface.co/Styxxxx/llama2_7b_lora-wmt16_translate_fien}, 
    \url{https://huggingface.co/Styxxxx/llama2_7b_lora-wmt16_translate_ruen}, 
    \url{https://huggingface.co/Styxxxx/llama2_7b_lora-wmt16_translate_tren}}
    \item Natural Language Understanding\footnote{\url{https://huggingface.co/Styxxxx/llama2_7b_lora-wnli}, 
    \url{https://huggingface.co/Styxxxx/llama2_7b_lora-sst2}, 
    \url{https://huggingface.co/Styxxxx/llama2_7b_lora-snli}, 
    \url{https://huggingface.co/Styxxxx/llama2_7b_lora-rte}, 
    \url{https://huggingface.co/Styxxxx/llama2_7b_lora-qnli}, 
    \url{https://huggingface.co/Styxxxx/llama2_7b_lora-cola}}
\end{itemize}

For evaluation, we used the first 1,000 samples from CNN/DM summarization~\cite{cnn}, the full test set of PubMedQA~\cite{jin2019pubmedqa}, and HumanEval~\cite{chen2021codex}. Performance was measured using ROUGE scores for summarization, accuracy for medical QA, and pass@1 accuracy for code generation.

\paragraph{Vision Tasks.} 
We use models fine-tuned on the same 20 tasks as \cite{wang2024localizing}: KMNIST~\cite{kmnist}, EMNIST~\cite{cohen2017emnist}, SVHN~\cite{svhn}, GTSRB~\cite{gtsrb}, FER2013~\cite{fer2013}, DTD~\cite{dtd}, EuroSAT~\cite{helber2019eurosat}, MNIST~\cite{deng2012mnist}, RenderedSST2~\cite{renderedsst2, renderedsst22}, Cars~\cite{cars}, PCAM~\cite{pcam}, RESISC45~\cite{resisc},
FashionMNIST~\cite{xiao2017fashion}, SUN397~\cite{xiao2016sun}, CIFAR100~\cite{cifar10}, Flowers102~\cite{flower102}, Food101~\cite{bossard2014food}, OxfordIIITPet~\cite{OxfordIIITPet}, CIFAR10~\cite{cifar10}, STL10~\cite{stl10}.

\subsection{Baselines} \label{app:baselines}

\begin{itemize}
    \item \textbf{Pre-trained}: Employs a pre-trained model for each task without adapting it to the downstream tasks.
    \item \textbf{Individual}: Fine-tunes distinct models for each task, providing the performance upperbound for task-specific performance.
    \item \textbf{Traditional MTL}: Fine-tunes a single model on all tasks, providing a baseline for multi-task learning.
    \item \textbf{Weight Averaging~\cite{izmailov2018averaging}}: Averages the weights of separately fine-tuned models for different tasks, serving as a simple baseline.
    \item \textbf{Task Arithmetic~\cite{mitchell2022editing}}: Creates a multi-task vector by adding individual task vectors, which are scaled by a coefficient ($\lambda$) and incorporated into the pre-trained model's parameters.
    \item \textbf{Fisher Merging~\cite{fisher}}: Uses the Fisher information matrix to determine the importance of model parameters, preserving crucial parameters for each task.
    \item \textbf{Ties-Merging~\cite{yadav2023ties}}: Merges models by applying techniques like pruning, parameter sign determination, and separate merging to generate a merged task vector ($\tau$), which is added to the original model’s parameters with a scaling factor ($\lambda$) tuned on a validation set.
    \item \textbf{AdaMerging~\cite{yang2023adamerging}}: Adapts merging coefficients at either the task or layer level by minimizing entropy over unlabeled test data, using this as a surrogate objective for model merging.
    \item \textbf{Concrete Merging~\cite{tang2023concrete}}: Utilizes a meta-learning framework to generate a concrete mask that mitigates task interference during the merging process.
    \item \textbf{Representation Surgery~\cite{yang2024representation}}: Aligns the representation of the merged model with those of the individual models while adjusting biases to ensure compatibility across tasks.
\end{itemize}

We used Fusion Bench~\cite{tangFusionBenchComprehensiveBenchmark2024} for evaluation of the vision tasks. We follow the experiment setup provided there. AdaMerging is run with the same setup as detailed in their paper, with a learning rate of 0.001, momentum values of (0.9, 0.999), a batch size of 16, and 500 iterations. Surgery is applied to the merged model from AdaMerging.

\subsection{Implementations} \label{app:impl}

On language benchmarks, with the initial solution being the merged model from task arithmetic, and \fwta is run for 10 iterations. On vision tasks, the initial solution is the merged model from AdaMerging, and \fwta runs for 3 iterations. For vision benchmarks, \fwam is run for 15 iterations with the pre-trained model as the initial solution.

For the discriminative language benchmark, 100 data samples from each of MNLI, QNLI, QQP, and RTE are randomly selected as calibration datasets. For generative language tasks, 100 samples are randomly drawn from the training splits of CNN/DM~\cite{cnn}, CodeAlpaca-20k~\cite{codealpaca}, and PubMedQA~\cite{jin2019pubmedqa}. For vision tasks, training samples are randomly drawn from the training splits of SUN397~\cite{sun2020adashare}, Stanford Cars~\cite{cars}, GTSRB~\cite{gtsrb}, and DTD~\cite{dtd}.

\subsection{Scaling Experiment Setups} \label{app:scaling}

For scaling experiments with irrelevant models, we evaluate performance on SUN397~\cite{sun2020adashare}, Stanford Cars~\cite{cars}, GTSRB~\cite{gtsrb}, and DTD~\cite{dtd}. The irrelevant models consist of the vision models listed in Appendix~\ref{app:benchmark}, excluding those fine-tuned on these four tasks. For scaling experiments with relevant models, we use all 20 vision tasks as evaluation benchmarks, progressively adding the corresponding fine-tuned models to the merging pool. We employ \fwam for these scaling experiments. To ensure a fair comparison, \fwm optimizes the merging coefficients using entropy loss on test samples, similar to Adamerging. Adamerging is run for 300 iterations in experiments with irrelevant models and 200 iterations in those with relevant models.

\end{document}